%% file: main.tex
\documentclass[letterpaper]{article} 

\usepackage[preprint]{aaai2027} 

\usepackage[hyphens]{url} 
\usepackage{graphicx} 
\usepackage{natbib} 
\usepackage{caption} 
\usepackage{amsmath}
\usepackage{amssymb}
\usepackage{booktabs}
\usepackage{multirow}
\usepackage{algorithm}
\usepackage{algorithmic}
\usepackage{xcolor}
\usepackage{subcaption}
\usepackage{enumitem}
\usepackage{amsthm}

\newtheorem{definition}{Definition}
\newtheorem{assumption}{Assumption}
\newtheorem{theorem}{Theorem}
\newtheorem{lemma}{Lemma}

\newcommand{\projectleadmark}{\textsuperscript{\ensuremath{\ddagger}}}
\newcommand{\equalmark}{%
    \textsuperscript{*}%
}

\newcommand{\correspondingmark}{%
    \textsuperscript{\ensuremath{\dagger}}%
}

\makeatletter

\newcommand{\teaserfigure}[1]{%
    \gdef\aaai@teaser{#1}%
}

\teaserfigure{}

\g@addto@macro\@maketitle{%
    \ifx\aaai@teaser\@empty
    \else
        \par
        \vskip 1.0ex
        {%
            \hsize\textwidth
            \linewidth\hsize
            \centering
            \aaai@teaser
            \par
        }%
        \vskip 0.8ex
    \fi
}

\title{
When Does Legacy Data Start to Help?\\
Emergent Transfer in Cross-Configuration Robot Learning
}

\author{
Tao Wang\textsuperscript{\rm 1,\rm 2}\equalmark,
Hudson Hou\textsuperscript{\rm 3}\equalmark,
Yingdong Hu\textsuperscript{\rm 2},
Yufeng Liu\textsuperscript{\rm 2,\rm 4},
Qinghai Li\textsuperscript{\rm 2},\\
Yingjie Jiang\textsuperscript{\rm 2},
Yingzhi Wang\textsuperscript{\rm 2,\rm 5},
Cheng Ma\textsuperscript{\rm 2},
Richard Wang\textsuperscript{\rm 2}\projectleadmark,
Yang Gao\textsuperscript{\rm 2,\rm 6}\correspondingmark
}

\affiliations{
\textsuperscript{\rm 1}
Huazhong University of Science and Technology,
\quad
\textsuperscript{\rm 2}
Spirit AI,
\quad
\textsuperscript{\rm 3}
Peking University
\\
\textsuperscript{\rm 4}
Shanghai Jiao Tong University,
\quad
\textsuperscript{\rm 5}
Harbin Institute of Technology,
\quad
\textsuperscript{\rm 6}
Tsinghua University

}

\begin{document}

\teaserfigure{%
    \includegraphics[
        width=0.92\textwidth
    ]{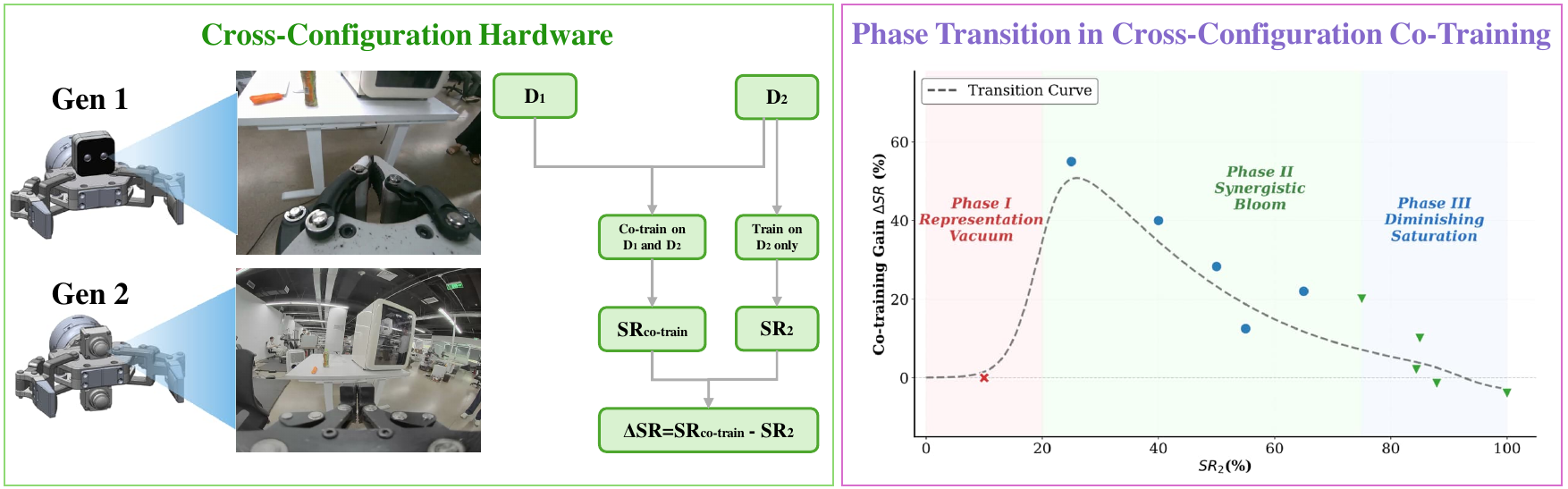}%
    \captionof{figure}{%
        The three-phase pattern of cross-configuration co-training.
        We compare policies trained on legacy data $D_1$,
        new-hardware data $D_2$, and their mixture.
        We denote the standalone success rate on the new configuration
        by $\mathrm{SR}_2$ and define the co-training gain as
        $\Delta\mathrm{SR}
        =
        \mathrm{SR}_{\mathrm{co\text{-}train}}
        -
        \mathrm{SR}_2$.
        As $\mathrm{SR}_2$ increases, legacy data is ineffective below
        the task-dependent transfer threshold $\tau(T)$, produces its
        largest gains just above the threshold, and yields diminishing
        returns at high baselines.
    }%
    \label{fig:three_phase_overview}%
}

\maketitle
\begingroup
\renewcommand{\thefootnote}{\fnsymbol{footnote}}
\footnotetext[1]{These authors contributed equally.}
\footnotetext[2]{Corresponding author: Yang Gao.}
\footnotetext[3]{Project lead: Richard Wang.}
\endgroup
\begin{abstract}
Robotic hardware evolves over time, but demonstration data is often tied
to a specific sensor and actuator configuration. This raises a practical
and underexplored question: when does legacy data begin to benefit an
upgraded robot? We study this question on a wheeled humanoid platform
across two hardware generations, where both the camera and gripper are
changed while the overall morphology remains fixed. Contrary to the
common assumption that more cross-configuration data is always helpful,
we observe a grokking-like transition: legacy data remains ineffective
until the upgraded configuration acquires a minimum level of task
competence, after which co-training gains rise sharply before diminishing
near saturation. We hypothesize that this task-dependent transition is
governed by a transfer threshold and characterize the resulting
three-phase pattern. Across real-robot manipulation tasks, we observe all
three phases: no measurable benefit at low competence
($10.0\% \rightarrow 10.0\%$), a sharp gain after crossing the threshold
($23.3\% \rightarrow 86.7\%$ on flower insertion), and diminishing
returns at high competence ($85.0\% \rightarrow 93.3\%$ on pen
insertion). We provide a theoretical account based on gradient alignment
and residual policy uncertainty, and derive a phase-aware rule for
deciding when to collect more new-hardware data and when to reuse legacy
demonstrations. We further validate this three-phase pattern on a mobile
dual-arm watering task, with results consistent with our predictions.
\end{abstract}

\input{sections/introduction}
\input{sections/related_work}

\input{sections/setup}
\input{sections/empirical_discovery}
\input{sections/analysis}
\input{sections/from_theory_to_practice}
\input{sections/conclusion}

%

\bibliography{references}

\clearpage
\appendix
\input{sections/appendix}

\end{document}

%% file: sections/introduction.tex
\section{Introduction}
\label{sec:introduction}

Driven by real-world physical interaction data, embodied AI has made steady progress in applications ranging from domestic manipulation and cooking to assistive care.
Recent large-scale robot learning and vision-language-action (VLA) models have shown that broad robot datasets and high-capacity policies can improve generalization across tasks, scenes, and embodiments~\citep{brohan2022rt1,brohan2023rt2,openx2024,octo2024,kim2024openvla}.
However, robots deployed in the real world inevitably undergo hardware iterations.
Sensors are upgraded to provide wider fields of view or higher sensing fidelity, while end-effectors are redesigned to improve dexterity. These changes alter the robot’s visual inputs and the low-level execution of gripper commands.
Each iteration therefore creates a practical data-reuse problem: large amounts of teleoperated demonstration data collected on legacy hardware may become difficult to reuse, even though such data is costly to acquire and often contains valuable task structure~\citep{argall2009survey,dasari2019robonet,ebert2022bridge}.

A natural response is to co-train abundant legacy demonstrations with a smaller amount of data collected on the upgraded robot.
This strategy is related to cross-embodiment and generalist robot learning, where policies are trained across heterogeneous robots, sensors, action spaces, and datasets
~\citep{openx2024,octo2024,kim2024openvla,zheng2025xvla,wang2024hpt}.
More directly, RoboNet and BridgeData show that prior robot data can reduce target-domain collection costs when combined with target data~\citep{dasari2019robonet,ebert2022bridge}.
Together, these studies establish that prior robot data can be useful, but they do not show whether it is beneficial from the outset after a hardware change.
Hardware iteration can introduce configuration-specific shifts in camera geometry, sensing fidelity, gripper control, and low-level action execution.
Transfer learning and domain adaptation suggest that source data may provide little benefit, or even interfere with learning, when the source and target domains are insufficiently aligned ~\citep{bendavid2010theory,rosenstein2005transfer,
wang2019characterizing,zhang2023survey}.
This leaves a central question unanswered: when does legacy data start to benefit an upgraded robot?

To answer this question, we compare policies trained only on new-hardware demonstrations with policies co-trained on legacy and new-hardware data.
Across several real-robot tasks and new-hardware data settings, we examine how co-training gain varies with the upgraded robot’s standalone success rate.
Flower insertion provides a particularly clear example of this pattern.
With 15.6 hours of high-quality new-hardware demonstrations, the standalone policy reaches only 23.3\% success, while co-training with 17 hours of legacy data raises performance to 86.7\%, a gain of 63.4 percentage points.
In contrast, with 4.3 hours of new-hardware data and a standalone success rate of 10.0\%, the same co-training strategy leaves performance unchanged at 10.0\% (Figure~1).
Thus, co-training gain can increase sharply after a limited improvement in standalone success.

Taken together, the results across tasks and data settings reveal a three-phase pattern.
At low standalone performance, legacy data provides no measurable benefit.
Once the new-hardware policy has learned sufficient task structure, the gain from co-training rises sharply.
As standalone performance approaches saturation, the marginal gain from legacy data decreases again.
We refer to the transition from ineffective to useful transfer in the intermediate regime as \emph{emergent transfer}.
This phenomenon is reminiscent of emergent abilities and grokking-like behavior~\citep{wei2022emergent,power2022grokking}.
However, the transition here occurs as the target configuration's standalone performance improves, rather than as model scale or training time increases.

We characterize this transition by a task-dependent transfer threshold, denoted by $\tau(T)$, defined as the minimum standalone performance above which legacy data is expected to improve performance on task $T$.
Below the threshold, the target policy has not yet learned sufficient task structure to use legacy demonstrations effectively.
Above it, gradients from legacy and new-hardware data can become better aligned, producing large gains while substantial room for improvement remains.
Near saturation, this room shrinks, and the marginal value of legacy data declines.
Our theoretical account is motivated by prior work on gradient
interference in joint learning and uncertainty-aware robot imitation
~\citep{yu2020gradient,pignat2019bayesian}.
Specifically, gradient alignment explains when legacy data begins to help, while residual policy uncertainty explains why its benefit diminishes near saturation.
We further connect the transfer threshold to a task-complexity estimate $H(T)$ and use the resulting model to guide new-hardware data collection.

Our contributions are summarized as follows:
\begin{enumerate}[topsep=0pt,itemsep=0pt,leftmargin=*]
    \item We identify and statistically support an emergent three-phase transfer pattern in cross-configuration robot learning: legacy data is ineffective at low target competence, highly beneficial at intermediate competence, and subject to diminishing returns near saturation.

    \item We model this transition using a task-dependent transfer threshold $\tau(T)$ and provide a theoretical account based on gradient alignment, residual policy uncertainty, and domain adaptation.
    
    \item We formulate a phase-aware data collection rule that collects new-hardware demonstrations until the transfer threshold is crossed and then introduces legacy data for co-training, reducing collection time from $8$ hours to $1.5$ hours on a held-out task.
\end{enumerate}

%% file: sections/related_work.tex
\section{Related Work}
\label{sec:related_work}

\paragraph{Robot demonstrations and data quality.}
Learning from demonstration studies how robot policies can be acquired from expert trajectories~\citep{argall2009survey,ross2011dagger,ho2016gail}.
More recent behavior-cloning and sequence-modeling methods improve visuomotor learning from offline demonstrations, including multimodal and long-horizon behaviors~\citep{mandlekar2021what,shafiullah2022behavior,chi2023diffusionpolicy}.
Because robot demonstrations are expensive to collect, their value depends not only on dataset size but also on data quality, task coverage, and the state distribution induced by the demonstrations~\citep{belkhale2023dataquality}.
We study the same data-efficiency problem after a hardware upgrade, where changes in sensing and end-effectors shift both the observation and action distributions.

\paragraph{Cross-robot learning and prior-data reuse.}
Large-scale robot learning increasingly trains policies across tasks, datasets, and robot platforms.
Open X-Embodiment aggregates data from diverse robot platforms, and the resulting RT-X models show that cross-robot training can improve downstream performance~\citep{openx2024}.
Octo and OpenVLA provide generalist policies that can be adapted to new observation spaces, action spaces, and platforms~\citep{octo2024,kim2024openvla}, while X-VLA and HPT introduce architectural mechanisms for handling embodiment heterogeneity~\citep{zheng2025xvla,wang2024hpt}.
RoboNet and BridgeData further show that data collected from other robots or domains can reduce target-domain data requirements~\citep{dasari2019robonet,ebert2022bridge}.
Together, these studies establish that prior robot data can support learning in a target domain.
Our focus is more specific: after a fixed hardware change, at what level of standalone target performance does legacy data begin to help?

\begin{figure*}[t]
\centering
\includegraphics[width=\textwidth]{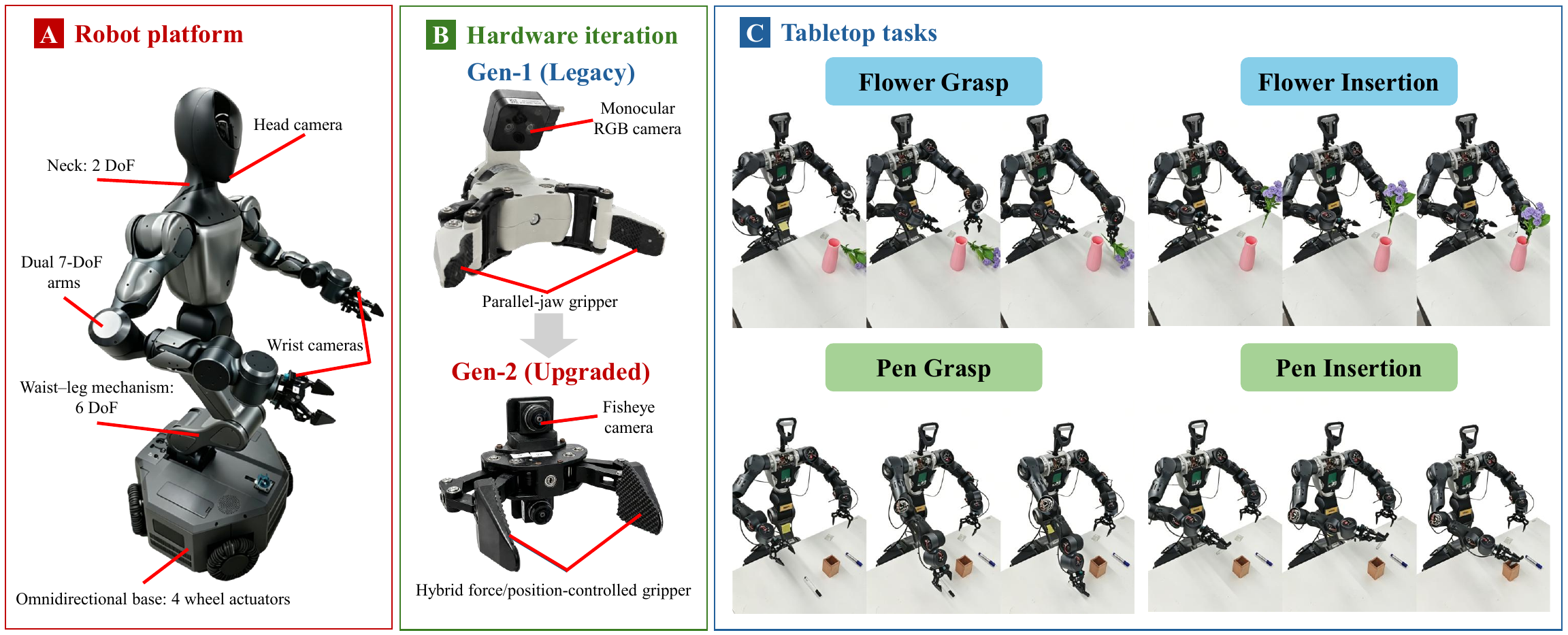}
\caption{
Robot platform, hardware iteration, and tabletop manipulation tasks.
(a) The robot comprises dual 7-DoF arms, a 6-DoF waist--leg mechanism, a 2-DoF neck, and a four-wheel omnidirectional base.
(b) Gen-2 replaces the monocular RGB cameras and position-controlled gripper of Gen-1 with fisheye cameras and a hybrid force/position-controlled gripper.
(c) The main experiments include flower grasp, flower insertion, pen grasp, and pen insertion.
}
\label{fig:platform_tasks}
\end{figure*}

\paragraph{Transfer learning, negative transfer, and emergent behavior.}
Transfer learning and domain adaptation provide a general framework for understanding when source-domain data improves target-domain performance~\citep{pan2010survey,bendavid2010theory}.
Classical bounds relate target error to source error, source--target divergence, and the error of the best shared hypothesis~\citep{bendavid2010theory}.
When the two domains are poorly aligned, source data may provide little benefit or lead to negative transfer~\citep{rosenstein2005transfer,wang2019characterizing,zhang2023survey}.
In our setting, the two hardware configurations share task semantics and arm kinematics but differ in camera observations and gripper control.
The transition in co-training gain resembles emergent abilities and grokking, although we observe it across levels of standalone target performance rather than model scale or optimization time~\citep{wei2022emergent,power2022grokking}.

\paragraph{Continual learning and hardware iteration.}
Continual robot learning studies how agents acquire new skills while retaining previously learned capabilities~\citep{kirkpatrick2017ewc,parisi2019continual,wan2024lotus}.
Hardware iteration also reuses prior experience, but presents a different problem.
In our setting, legacy and new-hardware data are available at the same time, so catastrophic forgetting is not the main concern.
We therefore study simultaneous co-training across hardware generations rather than sequential task acquisition.

%% file: sections/setup.tex
\section{Experimental Setup} 
\label{sec:experimental_setup}

\paragraph{Robot configurations.}
We study one hardware iteration on a wheeled humanoid robot with $26$ actuated DoF: four independently actuated omni wheels, dual $7$-DoF arms, a $6$-DoF waist--leg mechanism, and a $2$-DoF neck.
The mobile base is disabled for the main tabletop experiments, leaving $22$ active DoF, and is enabled for the mobile watering task.
The two generations share the same base, arm kinematics, and overall morphology, but differ in their visual sensing and gripper-control interfaces.
Gen-1 uses three monocular RGB cameras, each with a resolution of $640 \times 480$ and a field of view of $87^\circ \times 58^\circ$, together with an $8$\,cm position-controlled parallel-jaw gripper.
Gen-2 replaces them with three fisheye cameras, each with a resolution of $1920 \times 1536$ and a field of view of $120^\circ \times 120^\circ$, and an $8$\,cm hybrid force/position-controlled gripper equipped with wrist force--torque sensing.

\begin{table}[t]
\centering
\small
\small
\resizebox{\linewidth}{!}{%
\begin{tabular}{lccc}
\toprule
Task & Gen-1 legacy & Gen-2 early & Gen-2 refined \\
\midrule
Pen insertion    & $45.54$\,h & $18.63$\,h & $13.58$\,h \\
Flower insertion & $17.10$\,h & $4.31$\,h  & $15.60$\,h \\
\bottomrule
\end{tabular}
}
\caption{
Teleoperation data used in the main insertion experiments.
The refined column denotes a separately collected Gen-2 batch.
}
\label{tab:main_data_hours}
\end{table}

\paragraph{Tasks.}
The main study includes four fixed-base tabletop tasks: pen grasp, pen insertion, flower grasp, and flower insertion.
The grasping tasks are short-horizon and relatively tolerant to pose errors, whereas the insertion tasks require more precise alignment.
Flower insertion is the most difficult task because the robot must grasp a compliant stem and insert it into a narrow vase without bending the flower.
Object poses are randomized within a fixed workspace, and the robot base remains stationary throughout these trials.

\begin{table*}[t]
\centering
\small
\begin{tabular*}{\textwidth}{l@{\extracolsep{\fill}}ccccc}
\toprule
\textbf{Task} & \textbf{Config} & \textbf{Single} & \textbf{Co-train} & $\Delta$ & \textbf{$p$} \\
\midrule
Pen Grasp     & Gen-1 & 88.3\% & 86.7\% & $-$1.7 & 1.0 \\
              & Gen-2 & 100\% & 100\% & 0 & n.s. \\
Pen Insertion    & Gen-1 & 85.0\% & 86.7\% & +1.7 & 1.0 \\
              & Gen-2 & 56.7\% & 65.0\% & +8.3 & 0.46 \\
Flower Grasp  & Gen-1 & 100\% & 96.7\% & $-$3.3 & 0.50 \\
              & Gen-2 & 100\% & 100\% & 0 & n.s. \\
Flower Insertion & Gen-1 & 50.0\% & 78.3\% & +28.3 & 0.002 \\
              & \textbf{Gen-2} & \textbf{10.0\%} & \textbf{10.0\%} & \textbf{0} & \textbf{n.s.} \\
\bottomrule
\end{tabular*}
\vspace{1mm}
\caption{
Phase~I results using early Gen-2 data and Gen-1 legacy data.
The bold row marks the low-baseline Gen-2 flower-insertion case.
$p$ values are from two-sided Fisher's exact tests.
}
\label{tab:phase1}
\end{table*}

\paragraph{Policy and co-training.}
All policies are initialized from a pretrained $\pi_{0.5}$ vision-language-action model~\citep{black2025pi05} and fine-tuned using behavior cloning.
The policy takes multi-view RGB observations, proprioceptive states, and a language instruction as input, and predicts a $60$-step action chunk at each forward pass for closed-loop execution.
Both hardware generations use the same normalized arm-action representation, and the gripper output represents the desired opening position.
The difference between position-controlled and hybrid force/position-controlled gripper execution is handled by the corresponding configuration-specific low-level controller.
A single-configuration policy is trained using demonstrations from only one hardware generation, whereas a co-trained policy uses demonstrations from both generations.
Unless otherwise stated, co-training samples Gen-1 and Gen-2 data with equal probability, rather than in proportion to their respective demonstration hours.
The policy receives no explicit hardware-generation label.

\paragraph{Demonstration data.}
All demonstrations are collected through teleoperation at $30$\,Hz.
The central insertion experiments use one legacy Gen-1 dataset and multiple Gen-2 data batches, as summarized in Table~\ref{tab:main_data_hours}.
The quality-refined Gen-2 batches use stricter operator training, per-trajectory quality filtering, and tighter object-pose tolerances than the early batches.
Individual batches or their union are used depending on the experimental condition described in Section~\ref{sec:discovery}.

\paragraph{Evaluation and held-out validation.}
We report end-to-end task success rate, counting a trial as successful only when the full task is completed within the time limit without human intervention.
Each experimental condition is evaluated over $n=60$ real-robot trials, and co-training gain is defined as
\[
\Delta \mathrm{SR}
=
\mathrm{SR}_{\mathrm{co\text{-}train}}
-
\mathrm{SR}_{\mathrm{single}} .
\]
Differences between single-configuration and co-trained policies are evaluated using a two-sided Fisher's exact test, and Wilson $95\%$ confidence intervals are reported where applicable.
The mobile dual-arm watering task is held out from the experiments used to identify the three-phase pattern.
It uses the full $26$-DoF configuration and evaluates whether the phase-aware data collection rule extends to a mobile manipulation task.
Additional architecture, optimization, data-collection, and evaluation details are provided in the Appendix A.

%% file: sections/empirical_discovery.tex
\section{Empirical Discovery: A Three-Phase Model}
\label{sec:discovery}

Using the evaluation protocol described above, we analyze how co-training gain varies with the standalone success rate of each task--configuration pair.
The results reveal three regimes: no measurable gain at low baselines, large gains at intermediate baselines, and diminishing gains near saturation.

\begin{table}[t]
\centering
\small
\begin{tabular*}{0.9\linewidth}{@{}l@{\extracolsep{\fill}}cc@{}}
\toprule
\textbf{Task} & \textbf{Single} & \textbf{Co-train} \\
\midrule
Pen Insertion
& 71.7\%
& 98.3\% \\

\textbf{Flower Insertion}
& \textbf{23.3\%}
& \textbf{86.7\%} \\
\bottomrule
\end{tabular*}
\caption{
Phase~II results using the quality-refined Gen-2 datasets summarized in
Table~\ref{tab:main_data_hours}.
Co-training improves flower insertion from $23.3\%$ to $86.7\%$
and pen insertion from $71.7\%$ to $98.3\%$.
}
\label{tab:phase2}
\end{table}

\subsection{Discovery 1: Phase I Failure---Co-training Below the Threshold}
\label{sec:phase1}

We first co-train early Gen-2 data with Gen-1 legacy data and evaluate the resulting policy on both hardware generations.
Table~\ref{tab:phase1} shows the central negative case.
With only $4.31$h of early Gen-2 flower-insertion data, the upgraded hardware reaches a $10.0\%$ standalone baseline.
Co-training with $17.10$h of legacy Gen-1 flower data leaves performance unchanged at $10.0\%$.
Thus, in this low-baseline setting, legacy demonstrations provide no measurable benefit.

The contrast within the same run is important.
The amount of legacy data and the co-training procedure are fixed, yet different task--configuration pairs respond differently.
Gen-2 flower insertion remains at $10.0\%$, while Gen-1 flower insertion improves from $50.0\%$ to $78.3\%$.
This contrast motivates a task-dependent transfer threshold $\tau(T)$, below which legacy data provides no positive expected gain for task $T$.

\subsection{Discovery 2: Phase II---Large Gains at Intermediate Baselines}
\label{sec:phase2}

The Phase I failure raises a sharper question: does co-training require a high target baseline, or merely enough target-domain structure?
Using quality-refined Gen-2 data, flower insertion reaches a modest $23.3\%$ standalone success rate.
Co-training then lifts it to $86.7\%$, a gain of $63.4$ percentage points.
Pen insertion shows the same phase behavior, improving from $71.7\%$ to $98.3\%$.

After the new configuration crosses $\tau(T)$, co-training can produce large gains.
These results suggest that legacy data is most useful when the target policy has learned basic task structure but still has substantial room for improvement.

\subsection{Discovery 3: Phase III Saturation}
\label{sec:phase3}

At high standalone baselines, co-training gains shrink.
For Gen-2 pen insertion, the $71.7\%$ baseline gains $+26.6$ points, but after adding more Gen-2 data and reaching an $85.0\%$ standalone baseline, the same task gains only $+8.3$ points.
Near-ceiling Gen-1 grasping tasks show even smaller changes, including mild negative values that are not statistically significant.
Full saturation results are provided in Appendix~\ref{app:pen_insertion_details}.
This regime indicates that once the target policy already solves most of the task, legacy data has limited remaining uncertainty to reduce.

\subsection{Synthesis: The Three-Phase Model}
\label{sec:synthesis}

\begin{figure}[t]
\centering
\includegraphics[width=\columnwidth]{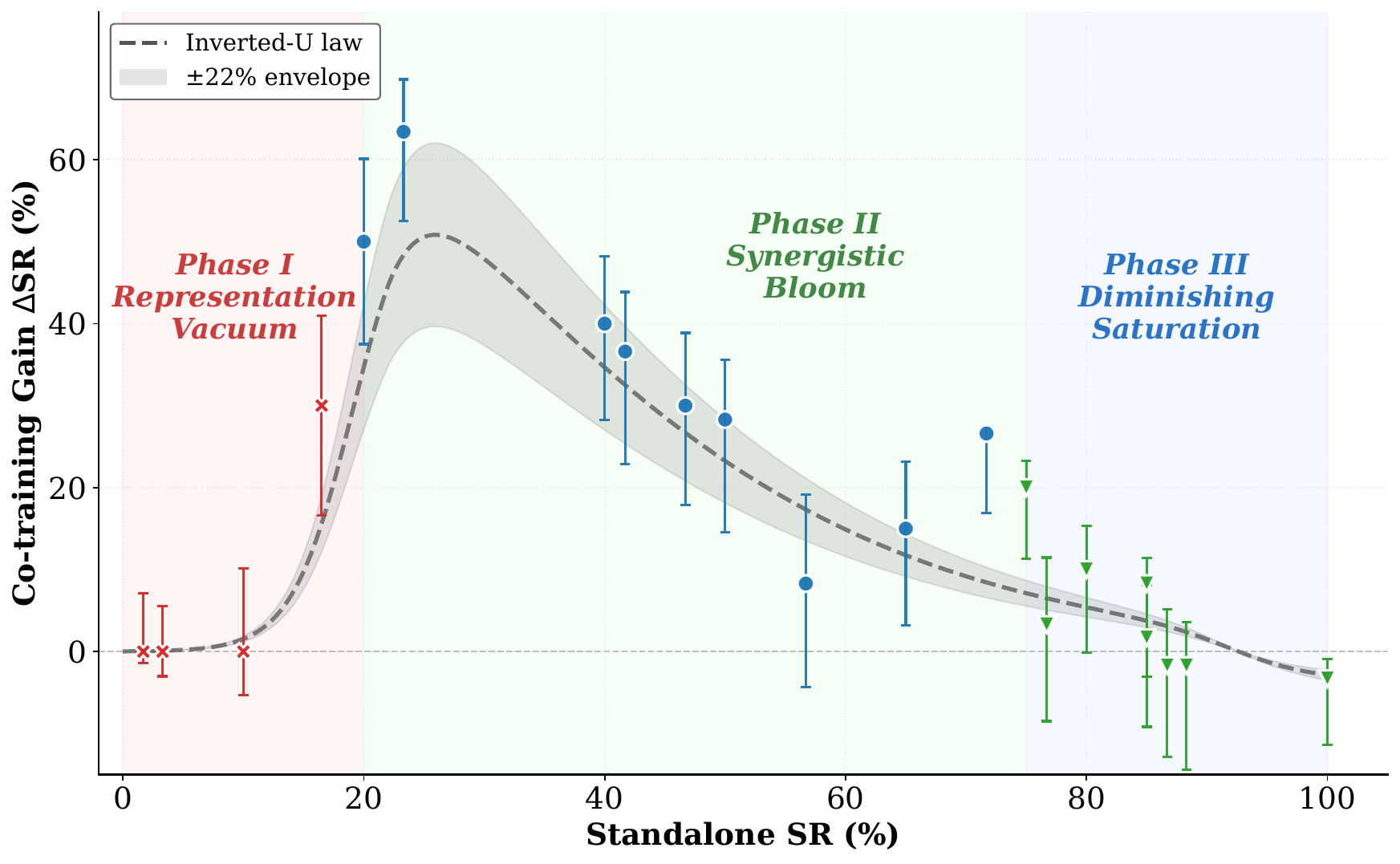}
\caption{
Phase transition in cross-configuration co-training.
Co-training gain $\Delta\mathrm{SR}$ is plotted against standalone success rate for real task-configuration pairs, with Wilson 95\% confidence intervals.
The dashed curve shows the inverted-U trend predicted by Theorem~\ref{thm:invertedU}: gains are near zero in Phase~I, peak in Phase~II, and diminish in Phase~III.
}
\label{fig:phase_curve}
\end{figure}

Figure~\ref{fig:phase_curve} summarizes the resulting pattern.
Co-training gain follows an inverted-U relationship with standalone success rate: it is near zero below $\tau(T)$, peaks after the target configuration has learned basic task structure, and diminishes near saturation.

\begin{table}[t] \centering \small \begin{tabular}{lccl} \toprule \textbf{Phase} & \textbf{Baseline} & \textbf{$\Delta$} & \textbf{Mechanism} \\ \midrule \textbf{I} & $<$15--20\% & $\approx 0$ & Representation vacuum \\ \textbf{II} & 20--75\% & +15 to +63 & Synergistic bloom \\ \textbf{III} & $>$75\% & +0 to +15 & Diminishing saturation \\ \bottomrule \end{tabular} \caption{The three-phase model.} \label{tab:three_phases} \end{table}

Two consequences follow.
First, the phase model is bidirectional: co-training is not merely source-to-target transfer.
In Table~\ref{tab:phase1}, the same co-trained policy that fails to lift Gen-2 flower insertion significantly improves Gen-1 flower insertion, while high-baseline Gen-1 tasks remain within the noise floor.
Second, data quality controls how quickly a task moves across phases, but does not bypass the threshold.
High-quality Gen-2 data reaches stronger standalone baselines with fewer hours; the response to co-training is then governed by the resulting phase.
Additional quality results are provided in Appendix~\ref{app:pen_insertion_details}.

%% file: sections/analysis.tex
\section{Why the Phase Structure Exists}
\label{sec:analysis}

The empirical results in Section~\ref{sec:discovery} suggest that legacy data becomes useful only after the target policy has learned basic task structure, and that its benefit declines as standalone performance approaches saturation.
We formalize this intuition with two components: a task-dependent transfer threshold that marks when legacy data begins to provide a positive training signal, and a residual-uncertainty term that captures the remaining room for improvement.

\paragraph{Stage structure as the order parameter.}
For a task $T$, let a trajectory decompose into latent stages
$\mathcal{Z}(T)$, such as \textsc{approach}, \textsc{grasp}, and
\textsc{insert}.
We define the stage decodability of the target policy as
\begin{equation}
\rho_c(T;\theta)
=
\max_g
\Pr_{(s,o)}
\left[
g(\phi_\theta(s,o))=z(s,o)
\right],
\label{eq:stage_decodability}
\end{equation}
where $\phi_\theta(s,o)$ is the internal representation and
$z(s,o)$ is the ground-truth task stage.
We assume that $\rho_c(T;\theta^{\mathrm{single}})$ is
non-decreasing with standalone success rate, so that
$\mathrm{SR}$ serves as an observable proxy for latent stage
decodability.

\begin{definition}[Task-dependent transfer threshold]
\label{def:tau}
For task $T$, the transfer threshold $\tau(T)$ is the smallest
standalone success rate above which co-training with legacy data
yields positive expected gain:
\begin{equation}
\tau(T)
=
\inf
\left\{
\mathrm{SR}:
\mathbb{E}
\left[
\Delta\mathrm{SR}\mid\mathrm{SR}
\right]>0
\right\}.
\label{eq:transfer_threshold}
\end{equation}
\end{definition}

\paragraph{Transfer-threshold transition.}
Let $g_S$ and $g_T$ denote the expected gradients from legacy and
target data.
We assume that the two configurations share a stage-conditional
objective up to configuration-specific residual error.
When stages are not decodable, legacy samples may be associated
with mismatched stage-conditional behavior, resulting in
non-positive expected gradient alignment.
Once the representation encodes stage identity, same-stage legacy
and target gradients can point toward a shared conditional optimum.

\begin{theorem}[Transfer-threshold transition]
\label{thm:threshold}
Under monotone coupling between stage decodability and standalone
success, a shared stage-conditional objective, and continuity of
expected gradient alignment, there exists a critical decodability
$\rho_{\mathrm{crit}}(T)$ and a corresponding transfer threshold $\tau(T)$ such that
\begin{equation}
\begin{aligned}
\mathbb{E}\langle g_S,g_T\rangle &\le 0,
&& \mathrm{SR}<\tau(T),\\
\mathbb{E}\langle g_S,g_T\rangle &>0,
&& \mathrm{SR}>\tau(T).
\end{aligned}
\label{eq:gradient_transition}
\end{equation}
Thus, legacy data provides no first-order improvement below the
threshold and becomes a positive training signal above it.
\end{theorem}

The result follows by decomposing both gradients over latent task
stages.
Near chance-level decodability, cross-stage terms cancel or
conflict; once stages are reliably decoded, positively aligned
same-stage terms receive greater weight.
Continuity then implies a crossing point in decodability, which
maps to $\tau(T)$ through the monotone relation with standalone
success.

\paragraph{Why gains peak at moderate baselines.}
A correctly routed legacy sample provides information in proportion
to the target policy's remaining within-stage uncertainty:
\begin{equation}
I_{\mathrm{legacy}}
=
\rho\,
\overline{H}_{\mathrm{within}}
-
\varepsilon_{\mathrm{dom}},
\label{eq:legacy_information}
\end{equation}
where $\varepsilon_{\mathrm{dom}}$ captures irreducible
configuration-specific mismatch.
We further assume that
$\overline{H}_{\mathrm{within}}
=\eta(1-\mathrm{SR})$ to first order.
Legacy data therefore becomes useful only after the transfer
threshold, while its potential value decreases as standalone
success approaches saturation.

\begin{theorem}[Inverted-U gain law]
\label{thm:invertedU}
Under the assumptions above, the expected co-training gain follows
\begin{equation}
\begin{aligned}
\mathbb{E}\!\left[\Delta\mathrm{SR}\mid\mathrm{SR}\right]
&=
\left[
\kappa(1-\mathrm{SR})-\delta(\mathrm{SR})
\right] \\
&\quad \times
\mathbb{1}\!\left[\mathrm{SR}>\tau(T)\right],
\end{aligned}
\label{eq:inverted_u}
\end{equation}
where $\kappa>0$ is task-dependent and $\delta(\mathrm{SR})$
captures configuration-conflict costs that may increase near
saturation.
\end{theorem}

\paragraph{Proof sketch.}
Below the transfer threshold, the expected alignment between legacy
and target gradients is non-positive, so legacy data provides no
first-order gain.
Above the threshold, a legacy sample contributes useful supervision
when it is associated with the correct latent stage.
The probability of correct stage association increases with stage
decodability $\rho$, while the value of a correctly routed sample is
proportional to the remaining within-stage uncertainty.
Under
$\overline{H}_{\mathrm{within}}=\eta(1-\mathrm{SR})$,
this contribution decreases as standalone success increases.
Subtracting the configuration-conflict cost
$\delta(\mathrm{SR})$ yields
Equation~\ref{eq:inverted_u}.

\paragraph{Connection to domain adaptation.}
The transfer-threshold transition can also be interpreted through
classical domain-adaptation bounds.
Below the threshold, the target representation has not yet formed
shared stage structure, so the error of the best joint hypothesis
across legacy and target configurations can remain large.
After the threshold is crossed, a shared stage-conditional hypothesis
becomes available, reducing the joint-error term.

Equation~\ref{eq:gradient_transition} accounts for the absence of
gain in Phase~I, while Equation~\ref{eq:inverted_u} predicts the
large gains at intermediate baselines and their decline near
saturation.
Full proofs and additional diagnostics are provided in
Appendix~\ref{app:theory}.

%% file: sections/from_theory_to_practice.tex
\section{From Theory to Practice}
\label{sec:practice}

The task-dependent transfer threshold $\tau(T)$ provides a practical criterion for hardware iteration.
Rather than deciding in advance how much legacy data to reuse, we first ask whether the new-hardware policy has reached the level of standalone performance at which legacy data begins to help.
This criterion motivates the phase-aware data collection rule below.

\paragraph{Phase-aware data collection rule.}
For each task $T$, we use the following rule:
\begin{enumerate}[topsep=2pt,itemsep=1pt,leftmargin=*]
    \item Estimate task difficulty from the task horizon $L(T)$ and spatial tolerance $\epsilon(T)$ as
    $H(T)=L(T)\log(1/\epsilon(T))$, and use it to estimate the transfer threshold $\hat{\tau}(T)$.
    \item Train a standalone policy on new-hardware data and measure its success rate $\mathrm{SR}_2(T)$.
    \item If $\mathrm{SR}_2(T)<\hat{\tau}(T)$, collect more new-hardware data; otherwise, co-train with legacy data.
\end{enumerate}

This rule is actionable because all required quantities are available before large-scale co-training.
Task horizon and tolerance can be estimated from the task specification, while $\mathrm{SR}_2(T)$ can be measured from a small standalone training run.
The rule also avoids aggregate decisions: a single hardware upgrade may contain Phase~I tasks that still need new-hardware data, Phase~II tasks that are ready for co-training, and Phase~III tasks where additional co-training has little marginal value.
Thus, the budget should be spent on moving below-threshold tasks across $\tau(T)$ rather than blindly mixing all legacy data into all tasks.

\subsection{Validation on a Mobile Dual-Arm Watering Task}
\label{sec:validation}

We validate the phase-aware data collection rule on a held-out mobile dual-arm watering task.
Unlike the fixed-base insertion experiments, watering enables the mobile base and uses the full $26$-DoF configuration.
The task combines navigation, dual-arm grasping, spout alignment, and pouring (Figure~\ref{fig:watering}).

We estimate the task complexity as $H(T)\approx44$, which is close to that of pen insertion despite the longer horizon because watering has a looser spatial tolerance.
Based on the complexity--threshold relation derived from the main tasks, we expect watering to have a relatively low transfer threshold.
Combining this estimate with a pilot standalone learning curve suggests that approximately $1$--$2$ hours of Gen-2 demonstrations should be sufficient to enter the high-gain regime.
We therefore evaluate three Gen-2 data budgets: $0.5$h, $1.5$h, and $8$h, chosen to probe the regions below, near, and well above the predicted crossing point.

\begin{figure}[t]
\centering
\includegraphics[width=\columnwidth]{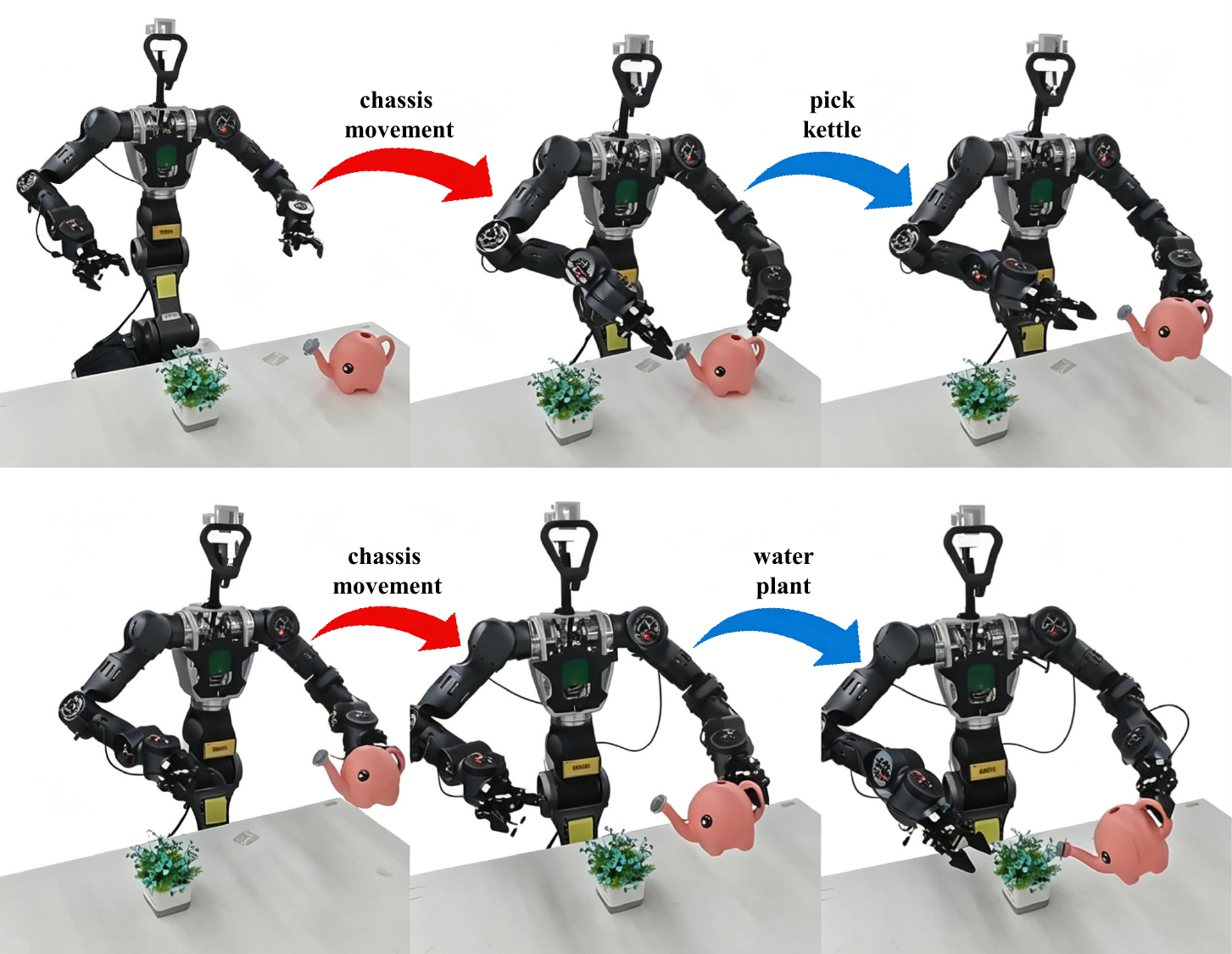}
\caption{
Mobile dual-arm watering task used for held-out validation.
The task consists of navigating to a kettle, grasping it with both arms, moving to a plant, aligning the spout, and pouring.
}
\label{fig:watering}
\end{figure}

We collect Gen-2 watering demonstrations at three data budgets: $0.5$h, $1.5$h, and $8$h. For all three conditions, the co-trained policy uses the same fixed $8$h Gen-1 watering dataset and the same training procedure, so only the amount of new-hardware data changes. Table~\ref{tab:practice} reports the resulting phase sweep. At $0.5$h, standalone performance remains below the transfer threshold, and co-training provides no measurable benefit. At $1.5$h, the task enters Phase~II, and co-training yields gains of $38$--$40$ percentage points. At $8$h, standalone performance is already high, and the gain from legacy data falls to the noise floor. The held-out task therefore exhibits the full three-phase pattern predicted by the model.

\begin{table}[t]
\centering
\small
\resizebox{\linewidth}{!}{%
\begin{tabular}{llcccc}
\toprule
\textbf{New-hw} & \textbf{Sub-stage} & \textbf{Single} & \textbf{Co-train} & $\Delta$ & \textbf{$p$} \\
\midrule
\multirow{2}{*}{$0.5$h} & Pick kettle & 3.3\%  & 3.3\%  & \textbf{0.0} & n.s. \\
                        & Water plant & 1.7\%  & 1.7\%  & \textbf{0.0} & n.s. \\
\midrule
\multirow{2}{*}{$1.5$h} & Pick kettle & 51.7\%  & 91.7\%  & \textbf{+40.0} & $1.5{\times}10^{-6}$ \\
                        & Water plant & 40.0\%  & 78.3\%  & \textbf{+38.3} & $3.5{\times}10^{-5}$ \\
\midrule
\multirow{2}{*}{$8$h}   & Pick kettle & 86.7\%  & 85.0\%  & \textbf{$-$1.7} & n.s. \\
                        & Water plant & 75.0\%  & 78.3\%  & \textbf{+3.3} & n.s. \\
\bottomrule
\end{tabular}
}
\caption{
Results on the held-out mobile dual-arm watering task.
Varying only the new-hardware data budget yields Phase~I at $0.5$h, Phase~II at $1.5$h, and Phase~III at $8$h.
$p$ values are from two-sided Fisher's exact tests.
}
\label{tab:practice}
\end{table}

The validation supports the data-efficiency claim behind the rule.
For this task, $1.5$h of new-hardware data is enough to reach the high-gain co-training regime, while collecting $8$h raises standalone performance but leaves little residual benefit for legacy data.
In other words, the rule does not recommend collecting as much new data as possible; it recommends collecting enough new data to make legacy demonstrations useful.
Additional details on the task-complexity estimate, sub-stage success
criteria, and controlled variables are provided in
Appendix~\ref{app:watering_validation}.

%% file: sections/conclusion.tex
\section{Conclusion and Discussion}
\label{sec:conclusion}

We studied when legacy demonstrations should be reused after a robot undergoes hardware iteration.
Our results challenge the common intuition that more cross-configuration data is uniformly helpful.
Instead, the benefit of legacy data depends on the phase of each task--configuration pair.
Below a task-dependent transfer threshold $\tau(T)$, co-training provides no measurable gain; after the threshold is crossed, it can produce large improvements, while the benefit diminishes near saturation.
On flower insertion, this contrast appears as $10.0\%\!\to\!10.0\%$ at a low standalone baseline and $23.3\%\!\to\!86.7\%$ after crossing the threshold.

The broader implication is that hardware iteration should not be treated as a binary data-reuse decision.
A practitioner should ask not only whether legacy data is compatible with the new hardware, but also whether the new configuration has reached the standalone performance at which legacy data begins to help.
Our theory characterizes this point with a task-dependent transfer threshold and explains the overall trend with an inverted-U gain law relating standalone performance to co-training gain.
The held-out watering task further shows how this three-phase pattern can guide the allocation of new-hardware data.
Thus, legacy data complements rather than replaces new-hardware data, becoming useful only after the new configuration crosses the transfer threshold.

\paragraph{Limitations and future work.}
Our experiments cover two hardware generations on a single wheeled-humanoid platform, with one VLA backbone and a limited set of manipulation tasks.
The task-complexity relation $\tau(T)\approx\alpha+\beta H(T)$ should therefore be viewed as a theoretical prediction supported by initial evidence, not a fully validated scaling law.
Future work should test the phase model across more robots, model families, and task types, and directly measure per-source gradient alignment during co-training.
Such studies would clarify whether the transfer threshold generalizes across embodied learning systems.
Ultimately, our results suggest a more measured view of data reuse in embodied AI: past experience helps not when it is abundant, but when the new system has learned enough structure to understand it.

%% file: sections/appendix.tex
\appendix

\section{Additional Experimental Details}
\label{app:experimental_details}

The main paper specifies the robot hardware, task suite, dataset
sizes, co-training protocol, and primary evaluation procedure.
This section provides the additional implementation and evaluation
details needed for reproducibility.

\subsection{Data Representation and Preprocessing}
\label{app:data_preprocessing}

Each trajectory contains multi-view RGB observations and
proprioceptive states.
Before training, camera images are resized to the model input
resolution.
The proprioceptive input consists of joint angles and gripper
aperture.

\subsection{Model Architecture and Optimization}
\label{app:model_training}

The $\pi_{0.5}$ backbone consists of a SigLIP vision encoder,
a Gemma-2B vision-language trunk, and an action-chunking
flow-matching head.
All model parameters are fine-tuned.

Unless otherwise stated, optimization uses AdamW with
$\beta_1=0.9$ and $\beta_2=0.95$.
The learning rate is linearly warmed up for $2{,}000$ steps to
$10^{-4}$ and then annealed to $10^{-5}$ using a cosine schedule.
Each policy is trained for $40{,}000$ steps.
The main experiments use $16$ NVIDIA A800 GPUs with $80$\,GB of
memory per GPU and a per-device batch size of $32$.

\subsection{Additional Evaluation Details}
\label{app:evaluation_protocol}

Partial task completion is counted as failure.
For example, grasping an object without completing the subsequent
insertion does not constitute a successful insertion trial.
The time limit is $30$\,s for grasping tasks and $60$\,s for
insertion tasks.

Trials are distributed across multiple evaluation periods.
Before each rollout, the robot is reset to a canonical starting
configuration.
Object placement is resampled within the predefined workspace,
while lighting and calibration conditions are held fixed within
each evaluation block.

\subsection{Statistical Testing}
\label{app:statistical_testing}

For each comparison between single-configuration training and
co-training, we construct a $2\times2$ success--failure contingency
table and apply a two-sided Fisher's exact test.
We treat $p<0.05$ as statistically significant, and comparisons
that do not reach this threshold are marked as n.s.

A negative value of $\Delta\mathrm{SR}$ is interpreted as evidence
of negative transfer only when the corresponding comparison is
statistically significant.
Otherwise, it is treated as a noise-level fluctuation.
Under this criterion, none of the negative gains observed in the
main experiments provides significant evidence of harmful transfer.

\subsection{Held-out Watering Evaluation}
\label{app:watering_details}

The two watering sub-stages are evaluated separately.
For the \emph{Pick kettle} sub-stage, success requires the robot to
navigate to the kettle and establish a stable dual-arm grasp.
For the \emph{Water plant} sub-stage, success requires the robot to
carry the kettle to the plant, align the spout with the pot, and
complete the pouring action.
Partial completion is counted as failure.

Across the new-hardware data-budget conditions reported in the main
paper, the Gen-1 legacy dataset, co-training sampling ratio,
optimization settings, and evaluation procedure are held fixed.
Only the amount of Gen-2 watering data is varied.

\section{Additional Empirical Results}
\label{app:empirical_details}

This section reports statistical details for the Phase~II results
and additional Gen-2 pen-insertion experiments that support the
data-quality and saturation analyses in Section~\ref{sec:discovery}.
The statistical protocol follows
Section~\ref{app:statistical_testing}.

\subsection{Phase II Statistical Significance}
\label{app:phase2_stats}

Table~\ref{tab:phase2_stats} reports the statistical significance
of the Phase~II improvements summarized in
Table~\ref{tab:phase2}.

\begin{table}[t]
\centering
\small
\begin{tabular}{lcc}
\toprule
\textbf{Task} & \textbf{$\Delta\mathrm{SR}$} & \textbf{$p$} \\
\midrule
Pen insertion (Gen-2)
    & +26.6 & $\approx 4.0\times10^{-5}$ \\
Flower insertion (Gen-2)
    & +63.4 & $\approx 1.8\times10^{-12}$ \\
\bottomrule
\end{tabular}
\caption{
Statistical significance of the Phase~II improvements reported in
Table~\ref{tab:phase2}.
}
\label{tab:phase2_stats}
\end{table}

\subsection{Pen-Insertion Data Quality and Saturation}
\label{app:pen_insertion_details}

Table~\ref{tab:pen_insertion_details} compares three Gen-2
pen-insertion data settings.
The refined batch reaches a higher standalone success rate than the
early batch despite containing fewer demonstration hours.
Combining the two batches further raises standalone performance to
$85.0\%$, while the co-training gain decreases to $+8.3$ points.
These results support both the data-quality observation and the
diminishing gain near saturation.

\begin{table}[t]
\centering
\small
\resizebox{\linewidth}{!}{%
\begin{tabular}{lcccc}
\toprule
\textbf{Gen-2 data} &
\textbf{Hours} &
\textbf{Single} &
\textbf{Co-train} &
\textbf{$\Delta\mathrm{SR}$} \\
\midrule
Early
    & $18.63$\,h & 56.7\% & 65.0\% & +8.3 \\
Refined
    & $13.58$\,h & 71.7\% & 98.3\% & +26.6 \\
Early + refined
    & $32.21$\,h & 85.0\% & 93.3\% & +8.3 \\
\bottomrule
\end{tabular}%
}
\caption{
Gen-2 pen-insertion results using the early, refined, and combined
datasets.
}
\label{tab:pen_insertion_details}
\end{table}

\section{Full Theoretical Details}
\label{app:theory}

This appendix extends the theoretical account in
Section~\ref{sec:analysis}.
For completeness, we restate the notation required by the proofs,
make the assumptions explicit, and provide the derivations of the
transfer-threshold transition and the inverted-U gain law.

\subsection{Notation and Assumptions}
\label{app:theory_defs}

Fix a task $T$ and a target hardware configuration $c$.
Let $\pi_\theta(a\mid s,o)$ denote a policy that maps proprioceptive
state $s$ and visual observation $o$ to action $a$.
A trajectory is decomposed into an ordered set of latent stages
\[
\mathcal{Z}(T)=\{z_1,\ldots,z_K\},
\]
where $z(s,o)\in\mathcal{Z}(T)$ denotes the ground-truth stage of
state--observation pair $(s,o)$.

Let $\phi_\theta(s,o)$ denote the policy's internal representation.
We restate stage decodability as
\begin{equation}
\rho_c(T;\theta)
=
\max_g
\Pr_{(s,o)}
\left[
g\!\left(\phi_\theta(s,o)\right)=z(s,o)
\right],
\label{eq:app_stage_decodability}
\end{equation}
where the maximum is taken over stage decoders $g$.

Write $\mathrm{SR}_c(T;\theta)$ for the standalone end-to-end
success rate on task--configuration pair $(T,c)$, and abbreviate
the success rate after standalone training as $\mathrm{SR}$.
For completeness, the task-dependent transfer threshold is
\begin{equation}
\tau(T)
=
\inf
\left\{
\mathrm{SR}:
\mathbb{E}
\left[
\Delta\mathrm{SR}\mid\mathrm{SR}
\right]>0
\right\}.
\label{eq:app_transfer_threshold}
\end{equation}

\begin{assumption}[Monotone coupling]
\label{ass:mono}
The stage decodability
$\rho_c(T;\theta^{\mathrm{single}})$ is non-decreasing with
standalone success rate $\mathrm{SR}$.
Thus, $\mathrm{SR}$ can be used as an observable proxy for latent
stage decodability.
\end{assumption}

\begin{assumption}[Shared stage-conditional objective]
\label{ass:share}
For each stage $z$, legacy and target demonstrations provide
supervision toward a shared stage-conditional objective, up to a
configuration-specific residual error
$\varepsilon_{\mathrm{dom}}$.
\end{assumption}

\begin{assumption}[Continuity of expected alignment]
\label{ass:continuity}
Expected legacy--target gradient alignment varies continuously with
stage decodability.
\end{assumption}

\subsection{Transfer-Threshold Transition}
\label{app:threshold_proof}

Let
\[
g_T(\theta)
=
\nabla_\theta\mathcal{L}_{\mathrm{target}}(\theta),
\qquad
g_S(\theta)
=
\nabla_\theta\mathcal{L}_{\mathrm{legacy}}(\theta)
\]
denote the expected gradients from target and legacy data.
Legacy data improves the target objective to first order when
$\langle g_S,g_T\rangle>0$.

\begin{lemma}[Stage structure controls gradient alignment]
\label{lem:align}
When stage decodability is near chance,
\[
\mathbb{E}\langle g_S,g_T\rangle\le 0.
\]
When stages are sufficiently decodable, same-stage legacy and target
gradients dominate, giving
\[
\mathbb{E}\langle g_S,g_T\rangle>0
\]
whenever learning signal remains.
\end{lemma}

\begin{proof}
Decompose the expected gradient from each data source by latent
stage:
\begin{equation}
g_X
=
\sum_{z\in\mathcal{Z}(T)}
\pi_z^X g_X^z,
\qquad
X\in\{S,T\},
\label{eq:app_gradient_decomposition}
\end{equation}
where $\pi_z^X$ is the probability of stage $z$ under source $X$
and $g_X^z$ is its stage-conditional gradient.

When the target representation does not encode stage identity,
legacy samples cannot be reliably associated with the corresponding
target stage.
The resulting cross-stage terms either average out or conflict under
Assumption~\ref{ass:share}, yielding non-positive expected
alignment.

Once stages become decodable, the diagonal terms
$\langle g_S^z,g_T^z\rangle$ receive greater weight.
Because legacy and target demonstrations share a stage-conditional
objective, these terms are non-negative and become strictly positive
when residual learning signal remains.
\end{proof}

\begin{proof}[Proof of the transfer-threshold transition]
Let
\begin{equation}
A(\rho)
=
\mathbb{E}
\left\langle
g_S(\theta(\rho)),
g_T(\theta(\rho))
\right\rangle
\label{eq:app_alignment_function}
\end{equation}
denote expected gradient alignment at stage decodability $\rho$.

By Lemma~\ref{lem:align}, $A(\rho)$ is non-positive near
chance-level decodability and positive when $\rho$ is sufficiently
large.
Assumption~\ref{ass:continuity} implies that $A(\rho)$ crosses zero
at some critical decodability
$\rho_{\mathrm{crit}}(T)$.
By Assumption~\ref{ass:mono}, this critical decodability corresponds
to a threshold in standalone-success space, denoted by $\tau(T)$.
Therefore,
\begin{equation}
\begin{aligned}
\mathbb{E}\langle g_S,g_T\rangle &\le 0,
&& \mathrm{SR}<\tau(T),\\
\mathbb{E}\langle g_S,g_T\rangle &>0,
&& \mathrm{SR}>\tau(T).
\end{aligned}
\label{eq:app_gradient_transition}
\end{equation}
Thus, legacy data provides no positive first-order training signal
below the transfer threshold and becomes useful after the threshold
is crossed.
\end{proof}

\subsection{Residual Uncertainty and the Inverted-U Gain Law}
\label{app:entropy_proof}

Treating the latent stage as a hidden variable, the stage-conditional
policy decomposition is
\begin{equation}
p(a\mid s,o)
=
\sum_{z\in\mathcal{Z}(T)}
p(z\mid s,o)\,
p(a\mid s,o,z).
\label{eq:app_stage_mixture}
\end{equation}

\begin{lemma}[Residual-entropy value of legacy data]
\label{lem:entropy}
Let
$\rho=\rho_c(T;\theta^{\mathrm{single}})$
be the target policy's stage decodability after standalone training.
We model the expected information contributed by a legacy sample as
\begin{equation}
I_{\mathrm{legacy}}
=
\rho\,
\overline{H}_{\mathrm{within}}
\left(\theta^{\mathrm{single}}\right)
-
\varepsilon_{\mathrm{dom}},
\label{eq:app_legacy_information}
\end{equation}
where
\begin{equation}
\overline{H}_{\mathrm{within}}(\theta)
=
\mathbb{E}_{(s,o)}
\left[
H\!\left[p_\theta(a\mid s,o,z)\right]
\right]
\label{eq:app_within_entropy}
\end{equation}
is the remaining within-stage policy uncertainty.
\end{lemma}

\begin{proof}
A legacy demonstration can contribute two forms of information:
stage-routing information and within-stage action supervision.
With probability $\rho$, the sample is associated with the correct
stage and contributes supervision proportional to the remaining
within-stage uncertainty.
With probability $1-\rho$, it is associated with a mismatched stage,
and its expected contribution does not consistently improve the
target stage-conditional policy under
Assumption~\ref{ass:share}.
The irreducible configuration-specific component is represented by
$\varepsilon_{\mathrm{dom}}$.
\end{proof}

\begin{assumption}[Linear capability--entropy coupling]
\label{ass:lincap}
To first order, residual within-stage uncertainty is proportional to
the unsolved fraction of the task:
\begin{equation}
\overline{H}_{\mathrm{within}}
\left(\theta^{\mathrm{single}}\right)
=
\eta(1-\mathrm{SR}),
\qquad
\eta>0.
\label{eq:app_entropy_coupling}
\end{equation}
\end{assumption}

\begin{assumption}[Bounded saturation interference]
\label{ass:sat}
Co-training introduces a non-negative configuration-conflict cost
$\delta(\mathrm{SR})$ that is small away from saturation and may
increase as $\mathrm{SR}\to1$.
\end{assumption}

\begin{proof}[Proof of the inverted-U gain law]
Below $\tau(T)$, the transfer-threshold result gives no positive
first-order legacy contribution.
We encode this inactive region using
$\mathbb{1}[\mathrm{SR}>\tau(T)]$.

Above the threshold, Lemma~\ref{lem:entropy} gives a useful
contribution proportional to the remaining within-stage uncertainty.
Substituting Assumption~\ref{ass:lincap} produces a term proportional
to $(1-\mathrm{SR})$.
On the active region, the positive stage-routing factor and fixed
domain-dependent constants are absorbed into the task-dependent
coefficient $\kappa>0$.
Subtracting the configuration-conflict cost from
Assumption~\ref{ass:sat} yields
\begin{equation}
\begin{aligned}
\mathbb{E}\!\left[
\Delta\mathrm{SR}\mid\mathrm{SR}
\right]
&=
\left[
\kappa(1-\mathrm{SR})
-
\delta(\mathrm{SR})
\right] \\
&\quad\times
\mathbb{1}
\left[
\mathrm{SR}>\tau(T)
\right].
\end{aligned}
\label{eq:app_inverted_u}
\end{equation}

The gain is inactive below the transfer threshold.
After the threshold is crossed, the remaining-uncertainty term is
largest at moderate standalone performance and decreases toward
saturation.
This produces the inverted-U gain pattern.
\end{proof}

\subsection{Training-Loss Diagnostics}
\label{app:loss_dynamics}

\begin{figure}[t]
\centering
\includegraphics[width=\columnwidth]
{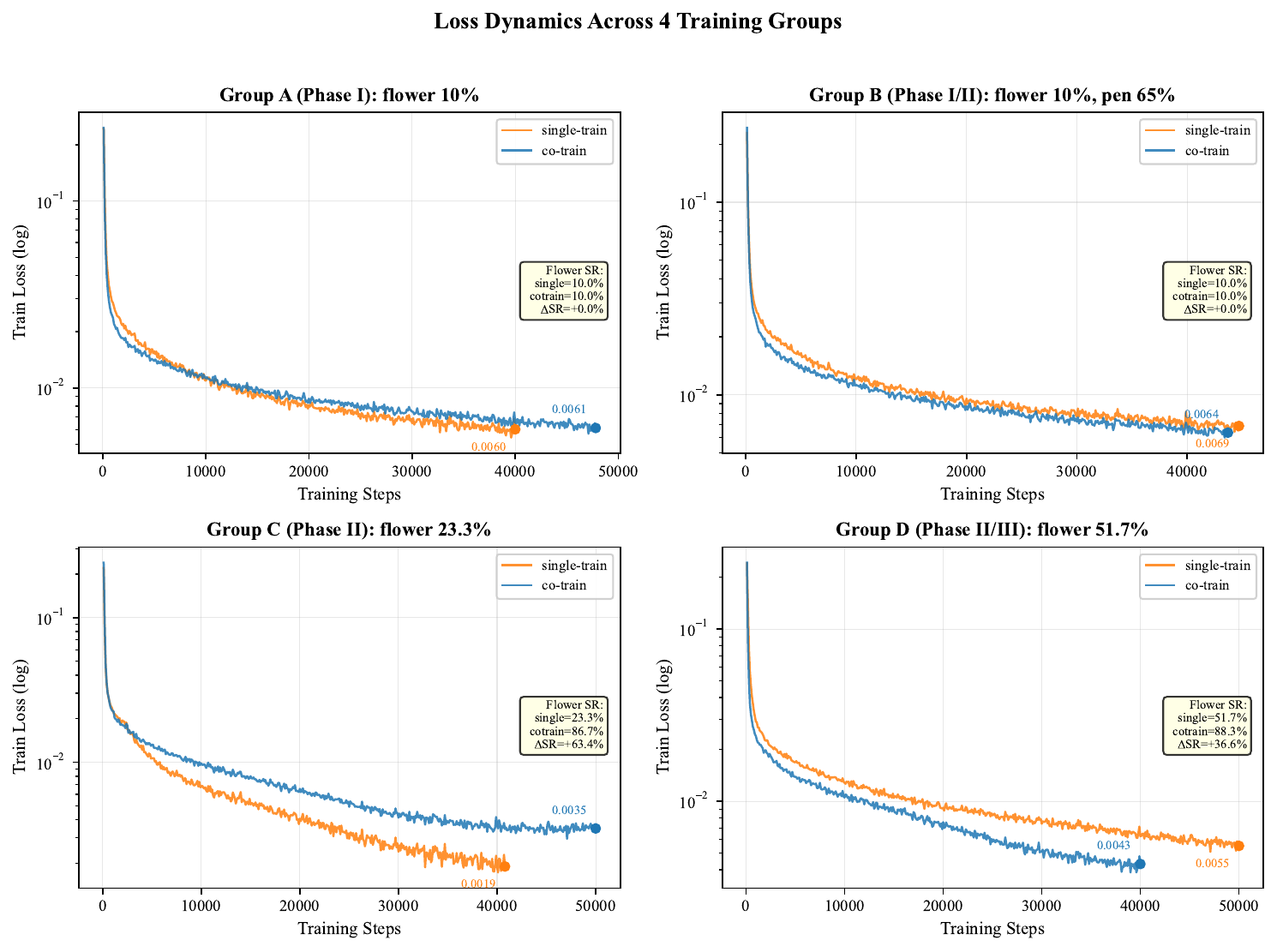}
\caption{
Training-loss dynamics across four matched standalone-training and
co-training conditions.
Phase~I groups show closely overlapping loss curves, while
higher-baseline groups show larger separation.
}
\label{fig:loss_dynamics}
\end{figure}

Direct per-source gradient alignment was not recorded during
training, so the loss trajectories in
Figure~\ref{fig:loss_dynamics} are used only as an indirect
diagnostic.
The Phase~I curves remain close, whereas larger separation appears
in the higher-baseline groups.
This pattern is consistent with, but does not directly establish,
the gradient-alignment account.

\subsection{Further Implications}
\label{app:corollaries}

\paragraph{Bidirectionality.}
Stage decodability and the transfer threshold are defined for each
task--configuration pair.
Consequently, two configurations included in the same co-training
run may benefit differently according to their respective standalone
performance and remaining uncertainty.
This accounts for the asymmetric improvements observed across
Gen-1 and Gen-2 evaluations without requiring a fixed source--target
direction.

\paragraph{Task-complexity prediction.}
We use
\begin{equation}
H(T)
=
L(T)\log\!\left(1/\epsilon(T)\right),
\label{eq:app_task_complexity}
\end{equation}
where $L(T)$ is the task horizon and $\epsilon(T)$ is its spatial
tolerance.
Longer horizons and tighter tolerances require more reliable stage
decodability, motivating the prediction that $\tau(T)$ is
non-decreasing with $H(T)$.

For the insertion tasks, pen insertion has approximately
$L=14$ and $H=42$, whereas flower insertion has approximately
$L=20$ and $H=60$.
Both tasks exhibit Phase~I near a $10\%$ standalone baseline and
enter the high-gain regime near the $20\%$ range.
These observations are consistent with the proposed monotonic
relation but are insufficient to establish a precise scaling law.
We therefore treat the relation between $H(T)$ and $\tau(T)$ as a
theoretical prediction rather than a validated empirical law.

\section{Additional Details for the Watering Validation}
\label{app:watering_validation}

This section provides the task-complexity estimate, sub-stage success
criteria, and controlled variables for the held-out watering
validation in Section~\ref{sec:validation}.

\subsection{Task-Complexity Estimate}
\label{app:watering_complexity}

The watering task includes mobile navigation, dual-arm grasping,
object transport, spout alignment, and pouring.
We estimate its task horizon as approximately $L(T)=21$ decision
steps.

The spout only needs to remain within a relatively loose region above
the pot opening for water to enter.
We therefore use a normalized spatial tolerance of
$\epsilon(T)\approx0.12$, corresponding to approximately
$12$\,cm relative to a $1$\,m reference scale.
The resulting task-complexity estimate is
\begin{equation}
\begin{aligned}
H(T)
&=
L(T)\log\!\left(1/\epsilon(T)\right) \\
&\approx
21\log(1/0.12)
\approx 44.
\end{aligned}
\label{eq:watering_complexity}
\end{equation}

This estimate is close to that of pen insertion
($H(T)\approx42$).
Although watering has a longer horizon, its spatial tolerance is
substantially looser, motivating the prediction that it should have
a relatively low transfer threshold.

\subsection{Success Criteria and Controlled Variables}
\label{app:watering_protocol}

The two watering sub-stages are evaluated separately.
For the \emph{Pick kettle} sub-stage, success requires the robot to
navigate to the kettle and establish a stable dual-arm grasp.
For the \emph{Water plant} sub-stage, success requires the robot to
carry the kettle to the plant, align the spout with the pot, and
complete the pouring action.
Partial completion is counted as failure.

Across the three new-hardware data budgets reported in
Table~\ref{tab:practice}, the $8$\,h Gen-1 legacy dataset,
co-training sampling ratio, optimization settings, and evaluation
procedure are held fixed.
Only the amount of Gen-2 watering data is varied.